\documentclass{article}
\usepackage{arxiv}
\usepackage[utf8]{inputenc} % allow utf-8 input
\usepackage[T1]{fontenc}    % use 8-bit T1 fonts
\usepackage{hyperref}       % hyperlinks
\usepackage{url}            % simple URL typesetting
\usepackage{booktabs}       % professional-quality tables
\usepackage{amsfonts}       % blackboard math symbols
\usepackage{nicefrac}       % compact symbols for 1/2, etc.
\usepackage{microtype}      % microtypography
\usepackage{graphicx}
\usepackage{natbib}
\usepackage{doi}

\usepackage[table,dvipsnames]{xcolor}
\usepackage{float}
\usepackage{tikz}

\PassOptionsToPackage{dvipsnames}{xcolor}
\definecolor{lightgray}{RGB}{220,220,220}
\definecolor{duckr}{HTML}{D81B60}
\definecolor{ducko}{HTML}{FE6100}
\definecolor{ducky}{HTML}{FFB000}
\definecolor{duckg}{HTML}{009E73}
\definecolor{duckb}{HTML}{0072B2}
\definecolor{duckp}{HTML}{9F94F4}

\usepackage{siunitx}

\usepackage{enumitem}

\usepackage{mathtools}

\usepackage{mathabx}

\newcommand*\circled[1]{\tikz[baseline=(char.base)]{
            \node[shape=circle,draw,inner sep=1.0pt] (char) {#1};}}

\usepackage{bm}

\usepackage{amsthm}
\newtheorem{theorem}{Theorem}[section]

\usepackage{algorithm}
\usepackage{algcompatible}
\usepackage{caption}
\newenvironment{tablealgo}[1][htb]
  {\captionsetup[table]{name= Algorithm}% Update table name
   \begin{table}[#1]%
  }{\end{table}}

\usepackage{subcaption}

\usetikzlibrary{ducks}
\usepackage[listings,skins,breakable]{tcolorbox}
{%
\end{tcolorbox}\endgroup
}
\newtcbox{\ducknotesmall}{breakable,enhanced jigsaw,nobeforeafter,tcbox raise base,boxrule=0.4pt,top=0mm,bottom=0mm,
right=0mm,left=4mm,arc=1pt,boxsep=2pt,before upper={\vphantom{dlg}},
colframe=Black!75!black,coltext=black,colback=Black!20,
overlay={
\begin{tcbclipinterior}
    \fill[Black!75!black] (frame.south west) rectangle ([xshift=4mm]frame.north west);
    \begin{scope}
        \randuck[xshift=0.1em, yshift=0.1em, scale=0.15];
    \end{scope}
\end{tcbclipinterior}
}}
\MakeRobust\ducknotesmall
\ifdefined
\fi

\title{Kernel-Based Ensemble Gaussian Mixture Probability Hypothesis Density Filter}

\author{
    {Dalton Durant}\thanks{Ph.D. Candidate} \\
	Department of Aerospace Engineering \& \\ Engineering Mechanics \\
	The University of Texas at Austin \\
	Austin, TX 78712 \\
	\href{mailto:ddurant@utexas.edu}{\texttt{ddurant@utexas.edu}} \\
    \And
	{Renato Zanetti}\thanks{Associate Professor} \\
	Department of Aerospace Engineering \& \\ Engineering Mechanics \\
	The University of Texas at Austin \\
	Austin, TX 78712 \\
	\href{mailto:renato@utexas.edu}{\texttt{renato@utexas.edu}} \\
}

\renewcommand{\shorttitle}{EnGM-PHD Filter}

\hypersetup{
pdftitle={Kernel-Based Ensemble Gaussian Mixture Probability Hypothesis Density Filter},
pdfsubject={},
pdfauthor={Dalton Durant, Renato Zanetti},
pdfkeywords={multi-target tracking, multi-target filtering, state estimation, nonlinear filtering},
}

\begin{document}
\maketitle

\begin{center}
    Code: \url{https://github.com/daltondurant/EnGM-PHD.git}
\end{center}

%%%%%%%%%%%%%%%%%%%%%%%
% ABSTRACT
%%%%%%%%%%%%%%%%%%%%%%%
\begin{abstract}In this work, a kernel-based Ensemble Gaussian Mixture Probability Hypothesis Density (EnGM-PHD) filter is presented for multi-target filtering applications. 
The EnGM-PHD filter combines the Gaussian-mixture-based techniques of the Gaussian Mixture Probability Hypothesis Density (GM-PHD) filter with the particle-based techniques of the Sequential Monte Carlo Probability Hypothesis Density (SMC-PHD) filter.
It achieves this by obtaining particles from the posterior intensity function, propagating them through the system dynamics, and then using Kernel Density Estimation (KDE) techniques to approximate the Gaussian mixture of the prior intensity function.
This approach guarantees convergence to the true intensity function in the limit of the number of components.
Moreover, in the special case of a single target with no births, deaths, clutter, and perfect detection probability, the EnGM-PHD filter reduces to the standard Ensemble Gaussian Mixture Filter (EnGMF).
In the presented experiment, the results indicate that the EnGM-PHD filter achieves better multi-target filtering performance than both the GM-PHD and SMC-PHD filters while using the same number of components or particles.
\end{abstract}

%%%%%%%%%%%%%%%%%%%%%%%
% KEYWORDS
%%%%%%%%%%%%%%%%%%%%%%%

%%%%%%%%%%%%%%%%%%%%%%%
% BODY TEXT
%%%%%%%%%%%%%%%%%%%%%%%
%%%%%%%%%%%%%%%%%%%%%%%%%%%%
\section{Introduction}
\label{sec:introduction}
%%%%%%%%%%%%%%%%%%%%%%%%%%%%

The Probability Hypothesis Density (PHD) filter~\cite{ref:mahler2003} provides an efficient approximation of the intractable multi-target Bayesian filter by propagating only its first-order moment---\emph{a.k.a} the intensity function. 
The intensity function describes the expected number of targets found in a unit volume of space, serving as a compact representation of the multi-target density. 
While it has reduced complexity compared to the full multi-target filter, there are unfortunately no closed form solutions of the PHD filter for nonlinear and non-Gaussian conditions.

Ref.~\cite{ref:vo2006} developed a Gaussian mixture implementation (GM-PHD), which approximates the intensity function as a weighted sum of Gaussian components, and~\cite{ref:vo2005} developed a sequential Monte Carlo implementation (SMC-PHD), which approximates the same intensity function as a weighted mixture of Dirac delta distributions using particles.
The GM-PHD filter provides an efficient closed-form solution under mild nonlinear conditions, but breaks down under larger nonlinearities~\cite{ref:vo2006, ref:wu2023}. 
On the other hand, the SMC-PHD filter can handle larger nonlinearities, but at the cost of being more computationally expensive~\cite{ref:vo2005, ref:wu2023}. 

Recent advances in single-target tracking (STT) have introduced kernel-based methods that merge the computational efficiency of Gaussian mixture filters with the nonlinear capabilities of particle filters. 
One such method, the Ensemble Gaussian Mixture Filter (EnGMF)~\cite{ref:anderson, ref:yun2022}, uses Kernel Density Estimation (KDE) to transform a set of particles into a Gaussian mixture. 
In this approach, each particle is treated as a Gaussian component with nonzero covariance and equal weight, producing a smooth and accurate approximation of the target’s probability density function (PDF). 
Although the EnGMF has demonstrated advantages in STT, no direct extension has been developed for multi-target tracking (MTT).

In the MTT literature, kernel-based approaches have seen limited application. Ref.~\cite{ref:wu2023} introduces a distributed cardinalized SMC-PHD filter that fuses particle sets from multiple nodes using KDE, while~\cite{ref:reifler2021} applies the EnGMF as the single-target filter to the multi-target generalized labeled multi-Bernoulli (GLMB) filter.  Despite these contributions, kernel-based approaches have not been developed to approximate the intensity function directly.

This work proposes the Ensemble Gaussian Mixture Probability Hypothesis Density (EnGM-PHD) filter, which uses the KDE approach of the EnGMF to produce a smooth and accurate approximation of the intensity function. 
The EnGM-PHD filter combines the Gaussian-mixture-based computational efficiency of the GM-PHD filter with the particle-based nonlinear capabilities of the SMC-PHD filter. 
It achieves this by obtaining particles from the posterior intensity function, propagating them through the system dynamics, and then using KDE techniques to approximate the Gaussian mixture of the prior intensity function.

The EnGM-PHD filter guarantees convergence to the true intensity function in the limit of the number of components.
Moreover, in the special case of a single target with no births, deaths, clutter, and perfect detection probability, the EnGM-PHD filter simplifies to the standard EnGMF, making the EnGM-PHD filter a straightforward MTT extension of the single-target EnGMF.

This work tests the GM-PHD, SMC-PHD, and EnGM-PHD filters in a MTT scenario with noisy measurements and considerable uncertainty in new target births. 
The scenario demonstrates the need to accurately approximate the intensity function in the presence of large measurement and model errors.
While using the same number of components or particles, the EnGM-PHD filter demonstrates better multi-target filtering performance than both the GM-PHD and the SMC-PHD filters, indicating a more accurate approximation of the intensity function.

The remainder of this work is structured as follows: Section~\ref{sec:background} provides background to the PHD, GM-PHD, and SMC-PHD filters. 
Section~\ref{sec:methodology} details the methodology for the EnGM-PHD filter.
Section~\ref{sec:results} compares the results of the GM-PHD, SMC-PHD, and EnGM-PHD filters in a simulated MTT scenario. 
Finally, Section~\ref{sec:conclusion} gives a brief discussion with concluding remarks.

%%%%%%%%%%%%%%%%%%%%%%%%%%%%
\section{Background and Motivation}
\label{sec:background}
%%%%%%%%%%%%%%%%%%%%%%%%%%%%

\subsection{Probability Hypothesis Density Filter}

The PHD filter~\cite{ref:mahler2003} is designed to provide an efficient approximation of the intractable multi-target Bayesian filter.
It is based on Random Finite Set (RFS) theory and avoids the combinatorial complexities of data association.
It is able to handle target survivals, spawns, births, and deaths, all while estimating the time-varying number of targets---in real time.

The PHD filter accomplishes these features by only propagating the first-order moment of the full multi-target density, \emph{a.k.a.} the intensity function. 
The intensity function, denoted by $v(\mathbf{x})$, serves as a compact representation of the multi-target density and describes the expected number of targets at state $\mathbf{x} \in \mathbb{R}^{n_\mathbf{x}}$, scattered randomly in space, according to a Poisson distribution. 
Its integral over any surveillance region $\mathrm{S} \subseteq \mathbb{R}^{n_\mathbf{x}}$ gives the expected number of targets $N$ in that region:
\begin{equation}
    N\ =\ \int_\mathrm{S} v(\mathbf{x}) \, \mathrm{d}\mathbf{x},
\end{equation}
where $\mathrm{d}\mathbf{x}$ is an infinitesimally small region of $\mathbf{x}$ in $\mathrm{S}$.

The PHD filter assumes:
\begin{enumerate}[label=(\roman*)]
    \setcounter{enumi}{0} % Starts from 1, which is (i)
    \item targets evolve independently and generate independent measurements,
    \item false alarms from clutter are Poisson distributed and are independent of measurements from actual targets,
    \item the predicted and updated multi-target RFS are also Poisson distributed.
\end{enumerate}

Given the above assumptions, the PHD filter's predicted intensity from time step ${k-1}$ to $k$ can be written as
\begin{equation}
    v(\mathbf{x}_k) = \int \it{\Phi}(\mathbf{x}_k, \mathbf{x}_{k-1})v(\mathbf{x}_{k-1}) \mathrm{d}\mathbf{x}_{k-1} + v_\gamma(\mathbf{x}_k),
    \label{eqn:phd-predict}
\end{equation}
and
\begin{equation}
    \begin{aligned}
        \it{\Phi}(\mathbf{x}_k, \mathbf{x}_{k-1}) 
        &= p_S(\mathbf{x}_{k-1})\phi(\mathbf{x}_k|\mathbf{x}_{k-1}) \\
        &+ v_\beta(\mathbf{x}_k|\mathbf{x}_{k-1}),
    \end{aligned}
    \label{eqn:Phi}
\end{equation}
where $p_S(\mathbf{x}_{k-1})$ is the probability of survival, $\phi(\cdot|\mathbf{x}_{k-1})$ is the state transition probability density, $v_\gamma(\cdot)$ is the birth intensity, and $v_\beta(\cdot|\mathbf{x}_{k-1})$ is the spawning intensity.

Let the measurement set at $k$ be denoted by $\mathbf{Z}_k \subset \mathbb{R}^{n_{\mathbf{z}}}$, which contains measurements and clutter. 
The intensity is then updated as
\begin{equation}
    \begin{aligned}
        v(\mathbf{x}_k | \mathbf{Z}_k) 
        &= \big ( 1 - p_D(\mathbf{x}_k) \big )v(\mathbf{x}_k) \\
        &+ \sum_{\mathbf{z} \in \mathbf{Z}_k} \frac{p_D(\mathbf{x}_k) p(\mathbf{z}|\mathbf{x}_k)v(\mathbf{x}_k)}{\kappa(\mathbf{z}) + \int p_D(\boldsymbol{\bm{\xi}}) p(\mathbf{z}|\boldsymbol{\bm{\xi}})v(\boldsymbol{\bm{\xi}})\mathrm{d}\boldsymbol{\bm{\xi}}},
    \end{aligned}
    \label{eqn:phd-update}
\end{equation}
where $p_D(\mathbf{x}_k)$ is the probability of detection, $p(\cdot|\mathbf{x}_k)$ is the likelihood of a single-target measurement given $\mathbf{x}_k$, and $\kappa(\cdot)$ is the clutter intensity. 
The clutter intensity can be further defined as ${\kappa(\cdot) \coloneqq \lambda_C p_C(\cdot)}$, where $\lambda_C$ is the Poisson average rate of uniform clutter per scan, and $p_C(\cdot)$ is the clutter density.

While the PHD filter has reduced complexity compared to the full multi-target filter, there are unfortunately no closed form solutions to the integrals of the intensity function for nonlinear and non-Gaussian conditions. 
Therefore,~\cite{ref:vo2006} developed a Gaussian mixture implementation for approximate linear Gaussian conditions, and~\cite{ref:vo2005} developed a particle implementation for more nonlinear and non-Gaussian conditions.

\subsection{Gaussian Mixture Implementation}

The GM-PHD filter~\cite{ref:vo2006} approximates the intensity function as a weighted sum of Gaussian mixture components.
The GM-PHD filter assumes (i)--(iii) and further assumes:
\begin{enumerate}[label=(\roman*)]
    \setcounter{enumi}{3} % Starts from 4, which is (iv)
    \item the dynamics and measurement models are linear,
    \item the survival and detection probabilities are independent of the target state,
    \item the birth and spawning intensities are Gaussian mixtures.
\end{enumerate}

Given the assumptions (i)--(vi), the GM-PHD approximates the PHD filter's prior intensity from~\eqref{eqn:phd-predict} as
\begin{equation}
    \begin{aligned}
        &\hspace{-2em}v(\mathbf{x}_{k}) \approx \underbrace{\sum_{i=1}^{J_{k-1}} p_S w_{k-1}^{(i)} \mathcal{N}(\mathbf{x}_k; \hat{\mathbf{x}}_{S,k|k-1}^{(i)}, \hat{\mathbf{P}}_{S,k|k-1}^{(i)})}_{\text{survived components}} \\
       +&\underbrace{\sum_{i=1}^{J_{k-1}}\sum_{j=1}^{J_{\beta,k}} w_{k-1}^{(i)}w_{\beta,k}^{(j)} \mathcal{N}(\mathbf{x}_k; \hat{\mathbf{x}}_{\beta,k|k-1}^{(i,j)}, \hat{\mathbf{P}}_{\beta,k|k-1}^{(i,j)})}_{\text{spawned components}} \\
       +&\underbrace{\sum_{i=1}^{J_{\gamma,k}} w_{\gamma,k}^{(i)} \mathcal{N}(\mathbf{x}_k; \hat{\mathbf{x}}_{\gamma,k}^{(i)}, \hat{\mathbf{P}}_{\gamma,k}^{(i)})}_{\text{birthed components}} \\ 
       &\phantom{\hspace{-2em}v(\mathbf{x}_{k})} = \sum_{i=1}^{J_{k|k-1}} w_{k|k-1}^{(i)} \mathcal{N}(\mathbf{x}_k; \hat{\mathbf{x}}_{k|k-1}^{(i)}, \hat{\mathbf{P}}_{k|k-1}^{(i)}),
    \end{aligned}
    \label{eqn:gmphd-prior-intensity}
\end{equation}
where it becomes a mixture of Gaussian mixtures of the survived, spawned, and birthed components. 
The hat notation $\hat{(\cdot)}$ indicates estimates. 
For example, $\hat{\mathbf{x}}$ represents the estimated mean and $\hat{\mathbf{P}}$ the estimated covariance. 
The last line in~\eqref{eqn:gmphd-prior-intensity} is to simplify the notation, where ${J_{k|k-1} = J_{k-1}(1 + J_{\beta,k}) + J_{\gamma,k}}$ is the total number of components. 

Now, given the Gaussian mixture approximation made in~\eqref{eqn:gmphd-prior-intensity}, the GM-PHD filter approximates the posterior intensity as
\begin{equation}
    \begin{aligned}
        &\hspace{-2em}v(\mathbf{x}_k | \mathbf{Z}_k) \approx (1 - p_D) v(\mathbf{x}_k) \\
        &+ \sum_{\mathbf{z} \in \mathbf{Z}_k} \sum_{i=1}^{J_{k|k-1}} w_{k|k}^{(i)} \mathcal{N}(\mathbf{x}_k; \hat{\mathbf{x}}_{k|k}^{(i)}, \hat{\mathbf{P}}_{k|k}^{(i)}).
    \end{aligned}
    \label{eqn:gmphd-posterior-intensity}
\end{equation}
Generally, for linear models, the means and covariances are updated individually following the update equations of the Kalman Filter (KF)~\cite{ref:kalman}.

However, for nonlinear models, assumption (iv) can be relaxed, leading to two extensions: EK-PHD and UK-PHD filters~\cite{ref:vo2006}.
The EK-PHD filter uses the Extended Kalman Filter (EKF)~\cite{ref:gelb} to individually predict and update components, while the UK-PHD filter uses the Unscented Kalman Filter (UKF)~\cite{ref:julier, ref:vandermerwe}.

After the update, the number of components grow combinatorially, which the GM-PHD filter manages through a pruning, merging, and capping scheme.
Multi-target state extraction is performed by first estimating the cardinality $\hat{N}_k$ as the sum of the remaining weights, and then selecting the $\hat{N}_k$ components with the highest weights as the estimated target states. 
Alternatively, they can be chosen to be the components whose weights exceed a prescribed threshold.~\cite{ref:vo2006}

The GM-PHD filter provides a smooth and accurate approximation to the true intensity function and guarantees convergence, to some specified accuracy, using a minimum number of mixture components~\cite{ref:clark2007_gmphd-convergence}. 
However, even though the GM-PHD filter is an efficient closed-form solution to the PHD filter for mildly nonlinear and non-Gaussian conditions, its multi-target filtering performance will break down under larger nonlinearities~\cite{ref:vo2006, ref:wu2023}. 

\subsection{Particle Implementation}

The SMC-PHD filter~\cite{ref:vo2005} uses a mixture of weighted Dirac delta functions, via particles, rather than a mixture of Gaussians to accommodate more nonlinear and non-Gaussian conditions. 
The SMC-PHD filter only needs assumptions (i)--(iii) and approximates the PHD filter's prior intensity from~\eqref{eqn:phd-predict} as
\begin{equation}
    \begin{aligned}
        &\hspace{-2em}v(\mathbf{x}_{k}) \approx \underbrace{\sum_{i=1}^{J_{k-1}} p_S(\bm{\mathcal{X}}_{k-1|k-1}^{(i)}) w_{k-1}^{(i)} \delta(\mathbf{x}_k - \bm{\mathcal{X}}_{S,k|k-1}^{(i)})}_{\text{survived particles}} \\
       +&\underbrace{\sum_{i=1}^{J_{k-1}}\sum_{j=1}^{J_{\beta,k}} w_{k-1}^{(i)}w_{\beta,k}^{(j)} \delta(\mathbf{x}_k - \bm{\mathcal{X}}_{\beta,k|k-1}^{(i,j)})}_{\text{spawned particles}} \\
       +&\underbrace{\sum_{i=1}^{J_{\gamma,k}} w_{\gamma,k}^{(i)} \delta(\mathbf{x}_k - \bm{\mathcal{X}}_{\gamma,k}^{(i)})}_{\text{birthed particles}} \\
       &\phantom{\hspace{-2em}v(\mathbf{x}_{k})} = \sum_{i=1}^{J_k} w_{k|k-1}^{(i)} \delta(\mathbf{x}_k - \bm{\mathcal{X}}_k^{(i)}) ,
    \end{aligned}
    \label{eqn:smcphd-prior-intensity}
\end{equation}
where $\delta(\cdot)$ represents the Dirac delta, which can be re-interpreted as the limit of a normal distribution with covariance matrix approaching zero.
The notation $\bm{\mathcal{X}}$ represents particles that are assumed independent and identically distributed (i.i.d.).

Given the Dirac mixture approximation made in~\eqref{eqn:smcphd-prior-intensity}, the SMC-PHD filter approximates the posterior intensity as
\begin{equation}
    \begin{aligned}
    v(\mathbf{x}_k | \mathbf{Z}_k)\ \approx\ \sum_{i=1}^{J_k}w_{k|k}^{(i)}\delta(\mathbf{x}_k-\bm{\mathcal{X}}_k^{(i)}), \\
    \end{aligned}
    \label{eqn:smcphd-posterior-intensity}
\end{equation}
where
\begin{equation}
    \begin{aligned}
        w_{k|k}^{(i)}
        &= \big ( 1 - p_D(\bm{\mathcal{X}}_k^{(i)}) \big )w_{k|k-1}^{(i)} \\
        &+ \sum_{\mathbf{z} \in \mathbf{Z}_k} \frac{p_D(\bm{\mathcal{X}}_k^{(i)}) p(\mathbf{z}|\bm{\mathcal{X}}_k^{(i)})w_{k|k-1}^{(i)} }{\kappa(\mathbf{z}) + \sum_{j=1}^{J_k} p_D(\bm{\mathcal{X}}_k^{(j)}) p(\mathbf{z}|\bm{\mathcal{X}}_k^{(j)})w_{k|k-1}^{(j)} }.
    \end{aligned}
    \label{eqn:smcphd-posterior-weights}
\end{equation}

After the update, particles are resampled to prevent degeneracy.  
This process involves sampling from an RFS or point process.
Once $J_k$ particles are resampled, they are assigned uniform weights summing to the estimated cardinality $\hat{N}_k$.
Similar to the GM-PHD filter, multi-target state extraction is performed by first estimating cardinality $\hat{N}_k$ as the sum of the weights. 
But, unlike the GM-PHD filter, the SMC-PHD filter obtains the estimated target states by using a clustering algorithm, generally either through \textit{k}-means or expectation maximization (EM), to group the resampled particles into $\hat{N}_k$ clusters.~\cite{ref:vo2005}

Although it does not provide a smooth approximation, the SMC-PHD filter does guarantee convergence to the true posterior intensity function as its number of particles increases~\cite{ref:johansen2006}.
However, even though it can handle larger nonlinearities, the SMC-PHD filter is generally more computationally expensive than the GM-PHD filter because it requires more particles to obtain the same multi-target filtering performance~\cite{ref:wu2023}.
In addition, its method of state extraction is less efficient as well~\cite{ref:wu2023}.

It is therefore the motivation of this work to develop a filter that combines the Gaussian-mixture-based computational efficiency of the GM-PHD filter with the particle-based nonlinear capabilities of the SMC-PHD filter.

%%%%%%%%%%%%%%%%%%%%%%%%%%%%%%%%%
\section{Kernel-Based Ensemble Gaussian Mixture Implementation}
\label{sec:methodology}
%%%%%%%%%%%%%%%%%%%%%%%%%%%%%%%%%

\subsection{EnGM-PHD Filter Recursion}

The following proposed EnGM-PHD filter only requires assumptions (i)--(iii) of the PHD filter, making it applicable for nonlinear and non-Gaussian conditions.
Its basic idea is to transform the prior intensity of particles into a continuous, smooth and accurate approximation using KDE.
This continuous approximation is represented by a Gaussian mixture, which is then updated using the GM-PHD filter equations. 

Recall the Dirac mixture approximation of the prior intensity from~\eqref{eqn:smcphd-prior-intensity}:
\begin{equation}
    \begin{aligned}
        &\hspace{-2em}v(\mathbf{x}_{k}) \approx  \sum_{i=1}^{J_k} w_{k|k-1}^{(i)} \delta(\mathbf{x}_k - \bm{\mathcal{X}}_k^{(i)}).
    \end{aligned}
    \label{eqn:engmphd-prior-intensity}
\end{equation}

To convert~\eqref{eqn:engmphd-prior-intensity} into a Gaussian mixture, the KDE method from~\cite{ref:silverman} is used to transform each particle into a Gaussian component with nonzero covariance and equal weight:
\begin{gather}
    v(\mathbf{x}_k)
    \approx \sum_{i=1}^{J_k}\frac{\hat{N}_k}{J_k}\mathcal{N}\big(\mathbf{x}_{k};\hat{\mathbf{x}}_{k|k-1}^{(i)},\frac{\beta_\text{Silv.}}{\hat{N}_k} \text{Cov} (\{\hat{\mathbf{x}}_{k|k-1}^{(i)} \}_{i=1}^{J_k} )\big), 
    \label{eqn:engmphd-prior-intensity-kde}
    \\
    \hat{N}_k = \sum_{i=1}^{J_k} w_{k|k-1}^{(i)},
    \\
    \beta_\text{Silv.}\ =\ \left ( \frac{4}{n_\mathbf{x} + 2} \right )^{\frac{2}{n_\mathbf{x} + 4}} J_k^{-\frac{2}{n_\mathbf{x} + 4}},
\end{gather}
where ${\hat{\mathbf{x}}_{k|k-1}^{(i)} = \bm{\mathcal{X}}_k^{(i)}}$ is the predicted mean of the $i$-th Gaussian component, $\text{Cov} (\{\hat{\mathbf{x}}_{k|k-1}^{(i)} \}_{i=1}^{J_k} )$ is the sample covariance matrix of all the particles, and $\beta_\text{Silv.}$ is the bandwidth parameter using Silverman’s Rule of Thumb~\cite{ref:silverman}.

For a single target, if the sampling distribution is Gaussian, Silverman’s Rule of Thumb provides the optimal bandwidth parameter by minimizing the mean integrated square error (MISE) as a performance criterion~\cite{ref:silverman}.
However, for non-Gaussian distributions, it may produce conservative estimates.

A conservative estimate is often preferable because it helps maintain filter stability and prevents overconfidence, which could otherwise lead to divergence.
Furthermore, Silverman’s Rule of Thumb is relatively simple to compute, which reduces computational cost over alternative bandwidth selection algorithms.

The above approach guarantees convergence to the true intensity function in the limit of the number of components. 
It inherits the GM-PHD filter's convergence to the posterior intensity through its Gaussian mixture approximation~\cite{ref:clark2007_gmphd-convergence} and it inherits the SMC-PHD filter's convergence to the prior intensity through its Dirac mixture approximation~\cite{ref:johansen2006}.

Furthermore, when transforming the prior intensity into a Gaussian mixture, the EnGM-PHD filter maintains convergence through KDE and Silverman's Rule of Thumb.
This is because KDE and Silverman's Rule of Thumb can approximate any function arbitrarily well in the limit, producing a smoothed representation of the particle distribution, and ultimately converging to exact Bayesian inference~\cite{ref:silverman, ref:yun2022}. 

Conventionally, $\beta_\text{Silv.}$ remains unscaled in the single-target case. 
However, in the multi-target case, the bandwidth parameter must be adjusted by dividing it by the estimated cardinality $\hat{N}_k$.
In~\eqref{eqn:engmphd-prior-intensity-kde}, this scaling ensures that the covariance is appropriately applied to the estimated number of targets.

Given the Gaussian mixture approximation made in~\eqref{eqn:engmphd-prior-intensity-kde}, the EnGM-PHD filter approximates the posterior intensity as
\begin{equation}
    \begin{aligned}
        &\hspace{-2em}v(\mathbf{x}_k | \mathbf{Z}_k) \approx \big (1 - p_D \big ) v(\mathbf{x}_k) \\
        &+ \sum_{\mathbf{z} \in \mathbf{Z}_k} \sum_{i=1}^{J_k} w_{k|k}^{(i)} \mathcal{N}(\mathbf{x}_k; \hat{\mathbf{x}}_{k|k}^{(i)}, \hat{\mathbf{P}}_{k|k}^{(i)}),
    \end{aligned}
    \label{eqn:engmphd-update}
\end{equation}
and,
\begin{gather}
    w_{k|k}^{(i)}\ =\ \frac{ p_D w_{k|k-1}^{(i)} \bm{\xi}_k^{(i)}}{ \kappa(\mathbf{z})\ +\  \sum_{j=1}^{J_k}p_D w_{k|k-1}^{(j)} \bm{\xi}_k^{(j)} } ,
    \\
    \hat{\mathbf{x}}_{k|k}^{(i)}\ =\ \hat{\mathbf{x}}_{k|k-1}^{(i)}\ +\ \mathbf{K}_k^{(i)}(\mathbf{z} - h(\hat{\mathbf{x}}_{k|k-1}^{(i)})) ,
    \\
    \hat{\mathbf{P}}_{k|k}^{(i)}\ =\ \hat{\mathbf{P}}_{k|k-1}^{(i)}\ -\ \mathbf{K}_k^{(i)}\mathbf{H}_k^{(i)}\hat{\mathbf{P}}_{k|k-1}^{(i)} ,
    \\
    \mathbf{K}_k^{(i)}\ =\ \hat{\mathbf{P}}_{k|k-1}^{(i)}(\mathbf{H}_k^{(i)})^T(\mathbf{S}_{k}^{(i)})^{-1} ,
    \\
    \mathbf{S}_{k}^{(i)}\ =\  \mathbf{H}_k^{(i)}\hat{\mathbf{P}}_{k|k-1}^{(i)}(\mathbf{H}_k^{(i)})^T\ +\ \mathbf{R},
    \\
    \hat{\mathbf{P}}_{k|k-1}^{(i)} = \frac{\beta_\text{Silv.}}{\hat{N}_k} \text{Cov} (\{\hat{\mathbf{x}}_{k|k-1}^{(i)} \}_{i=1}^{J_k} ),
    \\
    \mathbf{H}_k^{(i)} = \frac{\partial h(\mathbf{x})}{\partial \mathbf{x}} \Big |_{\mathbf{x}=\hat{\mathbf{x}}_{k|k-1}^{(i)}},
    \\
    \bm{\xi}_k^{(i)}\ =\ \mathcal{N}\left ( \mathbf{z}; h(\hat{\mathbf{x}}_{k|k-1}^{(i)}), \mathbf{S}_{k}^{(i)}\right ),
\end{gather}
where $(\cdot)^T$ represents the matrix transpose and $\mathbf{R}$ is the measurement noise covariance.

After the update, the EnGM-PHD filter resamples $J_k$ i.i.d. samples from the Gaussian mixture approximation of the posterior intensity following the steps in Algorithm~\ref{tbl:gmm-sampling}.

\begin{tablealgo}[!ht]
    \vspace{-0.05em}
    \caption{Sampling from a Gaussian mixture.}
    \label{tbl:gmm-sampling}
    \centering
    \hrulefill
    \begin{algorithmic}\\
    \State \textbf{Given:}\ $\{w^{(i)}, \hat{\mathbf{x}}^{(i)}, \hat{\mathbf{P}}^{(i)} \}_{i=1}^{J}$ and requested samples $J_\text{new}$.
    \State \textbf{Check:}\ weights are normalized and sum to 1.
    \vskip 0.5em
    \State For each $j$, where $j=1,\ldots,J_\text{new}$:
    \vskip 0.25em
    \State \circled{1} Draw a uniform random variable ${\mathrm{u}^{(j)} \sim \mathcal{U}(0,1)}$.
    \vskip 0.25em
    \State \circled{2} Select the smallest index $\ell$: $\sum_{i=1}^\ell w^{(i)} \geq \mathrm{u}^{(j)}$.
    \vskip 0.25em
    \State \circled{3} Draw a sample $\bm{\mathcal{X}}^{(j)} \sim \mathcal{N}(\mathbf{x}; \hat{\mathbf{x}}^{(\ell)}, \hat{\mathbf{P}}^{(\ell)})$.
    \vskip 0.5em
    \State \textbf{Output:}\ $\{ \bm{\mathcal{X}}^{(j)} \}_{j=1}^{J_\text{new}}$
    \end{algorithmic}
    \hrulefill

\end{tablealgo}

Similar to the SMC-PHD filter, once ${J_\text{new} = J_k}$ particles have been resampled, they are assigned uniform weights such that the sum of the weights is equal to the estimated cardinality $\hat{N}_k$.
Multi-target state extraction is then performed following the same procedure as the SMC-PHD filter. 
This method of state extraction adds a layer of complexity beyond that of the GM-PHD filter, which is deemed the price to pay to extend the applicability of the GM-PHD filter to large nonlinearities.

\subsection{Ensuring Uniform Weights in the Prior Intensity}

While the proposed KDE approach seems straightforward, assuming the weights are uniform comes with some challenges. 
For example, a situation may arise where a prior intensity has the form:
\begin{equation}
    \begin{aligned}
        &\hspace{-2em}v(\mathbf{x}_{k}) \approx \underbrace{\sum_{i=1}^{J_{k-1}} p_S(\bm{\mathcal{X}}_{k-1|k-1}^{(i)}) w_{k-1}^{(i)} \delta(\mathbf{x}_k - \bm{\mathcal{X}}_{S,k|k-1}^{(i)})}_{\text{survived particles}} \\
       +&\underbrace{\sum_{i=1}^{J_{\gamma,k}} w_{\gamma,k}^{(i)} \delta(\mathbf{x}_k - \bm{\mathcal{X}}_{\gamma,k}^{(i)})}_{\text{birthed particles}}. \\
    \end{aligned}
\end{equation}
The KDE step would be to construct a Gaussian mixture for each Dirac mixture: 
\begin{equation}
    \begin{aligned}
        &\hspace{-1em}v(\mathbf{x}_{k}) \approx \\
        &\underbrace{\sum_{i=1}^{J_{k-1}} \frac{\hat{N}_{S,k}}{J_{k-1}}\mathcal{N}\big(\mathbf{x}_{k};\hat{\mathbf{x}}_{S,k|k-1}^{(i)},\frac{\beta_{S,\text{Silv.}}}{\hat{N}_{S,k}} \text{Cov} (\{\hat{\mathbf{x}}_{S,k|k-1}^{(i)} \}_{i=1}^{J_{k-1}} )\big)}_{\text{KDE constructed survived mixture}} \\
       +&\underbrace{\sum_{i=1}^{J_{\gamma,k}} \frac{\hat{N}_{\gamma,k}}{J_{\gamma,k}}\mathcal{N}\big(\mathbf{x}_{k};\hat{\mathbf{x}}_{\gamma,k}^{(i)},\frac{\beta_{\gamma,\text{Silv.}}}{\hat{N}_{\gamma,k}} \text{Cov} (\{\hat{\mathbf{x}}_{\gamma,k}^{(i)} \}_{i=1}^{J_{\gamma,k}} )\big)}_{\text{KDE constructed birthed mixture}}. \\
    \end{aligned}
\end{equation}
However, despite being transformed into Gaussian mixtures, the weights in the above equation remain non-uniform. 
This imbalance can lead to issues where the birth model overwhelms the surviving components, resulting in a poor representation of the prior intensity.

To mitigate this, i.i.d. samples are drawn from the combined distribution and used to reconstruct a single Gaussian mixture with uniform weights, ensuring a more balanced representation.
This procedure is outlined in Algorithm~\ref{tbl:dual-gmm-sampling}, resulting in the following prior intensity: 
\begin{gather}
    v(\mathbf{x}_k)
    \approx \sum_{i=1}^{J_k'}\frac{\hat{N}_k'}{J_k'}\mathcal{N}\big(\mathbf{x}_k;\hat{\mathbf{x}}_{k|k-1}'^{(i)},\frac{\beta_\text{Silv.}'}{\hat{N}_k'} \text{Cov} (\{\hat{\mathbf{x}}_{k|k-1}'^{(i)} \}_{i=1}^{J_k'} )\big), 
    \\
    \hat{N}_k' = \hat{N}_k + \hat{N}_{\gamma,k},
    \\
    \beta_\text{Silv.}'\ =\ \left ( \frac{4}{n_\mathbf{x} + 2} \right )^{\frac{2}{n_\mathbf{x} + 4}} {J_k'}^{-\frac{2}{n_\mathbf{x} + 4}},
\end{gather}
where ${J_k' = J_{k-1}+J_{\gamma,k}}$ is the total number of components and ${ \hat{\mathbf{x}}_{k|k-1}'^{(i)} = \bm{\mathcal{X}}_k'^{(i)} }$ are the new samples.
The above Gaussian mixture ensures that the prior intensity has uniform weights and is now ready for the update step.

\begin{tablealgo}[!ht]
    \vspace{1em}
    \caption{Sampling from two Gaussian mixtures.}
    \label{tbl:dual-gmm-sampling}
    \centering
    \hrulefill
    \begin{algorithmic}\\
    \State \textbf{Given:}\ $\{w_1^{(i)}, \hat{\mathbf{x}}_1^{(i)}, \hat{\mathbf{P}}_1^{(i)} \}_{i=1}^{J_1}$, $\{w_2^{(i)}, \hat{\mathbf{x}}_2^{(i)}, \hat{\mathbf{P}}_2^{(i)} \}_{i=1}^{J_2}$, and requested samples $J_\text{new}$.
    \State \textbf{Precompute:}\ $W_1 = \sum_{i=1}^{J_1} w_1^{(i)}$, $W_2 = \sum_{i=1}^{J_2} w_2^{(i)}$, and $Z = W_1 + W_2$.
    \vskip 0.5em
    \State For each $j$, where $j=1,\ldots,J_\text{new}$:
    \vskip 0.25em
    \State \circled{1} Draw a uniform random variable ${\mathrm{u}^{(j)} \sim \mathcal{U}(0,1)}$.
    \vskip 0.25em
    \State \circled{2} If $\mathrm{u}^{(j)} < \frac{W_1}{Z}$, then select first mixture. Otherwise, select the second.
    \vskip 0.25em
    \State \circled{3} Within the selected mixture, use Algorithm~\ref{tbl:gmm-sampling} to draw a sample $\bm{\mathcal{X}}^{(j)}$.
    \vskip 0.5em
    \State \textbf{Output:}\ $\{ \bm{\mathcal{X}}^{(j)} \}_{j=1}^{J_\text{new}}$
    \end{algorithmic}
    \hrulefill
    \vspace{1em}
\end{tablealgo}

\subsection{Relation to the Ensemble Gaussian Mixture Filter}

\begin{theorem}(Reduction of the EnGM-PHD filter to the EnGMF)
    In the special case of a single target with no births, deaths, clutter, and perfect detection probability, the EnGM-PHD filter equations simplify exactly to those of the standard EnGMF from~\cite{ref:yun2022}.
\end{theorem}

\begin{proof}
    If there is a single target, then the estimated cardinality is ${\hat{N}=1}$, and the probability of survival is assumed to be perfect. 
    Moreover, in the absence of target births, the prior intensity reduces to a prior PDF:
    \begin{equation}
        \begin{aligned}
            v(\mathbf{x}_k)
            &\approx \sum_{i=1}^{J_k}\frac{1}{J_k}\mathcal{N}\big(\mathbf{x}_{k};\hat{\mathbf{x}}_{k|k-1}^{(i)},\beta_\text{Silv.} \text{Cov} (\{\hat{\mathbf{x}}_{k|k-1}^{(i)} \}_{i=1}^{J_k} )\big) \\
            &= p(\mathbf{x}_k).
            \label{eqn:engmf-prior}
        \end{aligned}
    \end{equation}
    Here, the bandwidth parameter $\beta_\text{Silv.}$ is given by
    \begin{equation}
        \beta_\text{Silv.}\ =\ \left ( \frac{4}{n_\mathbf{x} + 2} \right )^{\frac{2}{n_\mathbf{x} + 4}} J_k^{-\frac{2}{n_\mathbf{x} + 4}}.
    \end{equation}
    Furthermore, in the absence of deaths, clutter, and with perfect detection probability, the posterior intensity simplifies to a posterior PDF:
    \begin{equation}
        \begin{aligned}
            v(\mathbf{x}_k | \mathbf{z}_k) 
            &\approx \sum_{i=1}^{J_k} w_{k|k}^{(i)} \mathcal{N}(\mathbf{x}_k; \hat{\mathbf{x}}_{k|k}^{(i)}, \hat{\mathbf{P}}_{k|k}^{(i)}) \\
            &= p(\mathbf{x}_k | \mathbf{z}_k).
        \end{aligned}
    \end{equation}
    The corresponding update equations are
    \begin{gather}
        w_{k|k}^{(i)}\ =\ \frac{ w_{k|k-1}^{(i)} \bm{\xi}_k^{(i)}}{ \sum_{j=1}^{J_k}w_{k|k-1}^{(j)} \bm{\xi}_k^{(j)} } ,
        \\
        \hat{\mathbf{x}}_{k|k}^{(i)}\ =\ \hat{\mathbf{x}}_{k|k-1}^{(i)}\ +\ \mathbf{K}_k^{(i)}(\mathbf{z}_k - h(\hat{\mathbf{x}}_{k|k-1}^{(i)})) ,
        \\
        \hat{\mathbf{P}}_{k|k}^{(i)}\ =\ \hat{\mathbf{P}}_{k|k-1}^{(i)}\ -\ \mathbf{K}_k^{(i)}\mathbf{H}_k^{(i)}\hat{\mathbf{P}}_{k|k-1}^{(i)} ,
        \\
        \mathbf{K}_k^{(i)}\ =\ \hat{\mathbf{P}}_{k|k-1}^{(i)}(\mathbf{H}_k^{(i)})^T(\mathbf{S}_{k}^{(i)})^{-1} ,
        \\
        \mathbf{S}_{k}^{(i)}\ =\  \mathbf{H}_k^{(i)}\hat{\mathbf{P}}_{k|k-1}^{(i)}(\mathbf{H}_k^{(i)})^T\ +\ \mathbf{R},
        \\
        \hat{\mathbf{P}}_{k|k-1}^{(i)} = \beta_\text{Silv.} \text{Cov} (\{\hat{\mathbf{x}}_{k|k-1}^{(i)} \}_{i=1}^{J_k} ),
        \\
        \mathbf{H}_k^{(i)} = \frac{\partial h(\mathbf{x})}{\partial \mathbf{x}} \Big |_{\mathbf{x}=\hat{\mathbf{x}}_{k|k-1}^{(i)}},
        \\
        \bm{\xi}_k^{(i)}\ =\ \mathcal{N}\left ( \mathbf{z}; h(\hat{\mathbf{x}}_{k|k-1}^{(i)}), \mathbf{S}_{k}^{(i)}\right ).
        \label{eqn:engmf-posterior}
    \end{gather}
    Equations~\eqref{eqn:engmf-prior}--\eqref{eqn:engmf-posterior} are precisely the prior and posterior PDFs of the EnGMF. Additionally, the resampling procedure in Algorithm~\ref{tbl:gmm-sampling} \ is identical to that in the EnGMF. 
    Therefore, under these conditions, the EnGM-PHD filter recursion reduces exactly to the EnGMF.
\end{proof}

%%%%%%%%%%%%%%%%%%%%%%%%%%%%%%%%%
\section{Numerical Experiment}
\label{sec:results}
%%%%%%%%%%%%%%%%%%%%%%%%%%%%%%%%%

This experiment tests the GM-PHD, SMC-PHD, and EnGM-PHD filters in a scenario that is concerned with tracking two targets.
Both targets travel at constant velocities, and cross at the exact same point in time and space. 
They are being tracked by a radar sensor, placed at the origin, which produces noisy range, azimuth, and elevation measurements, recorded at a rate of $\mathrm{d}t$. The targets remain trackable at all times. 

\subsection{Model Parameters}

The state vector is six-dimensional:
\begin{equation}
        \mathbf{x} = [\mathbf{r}_x, \mathbf{r}_y, \mathbf{r}_z, \mathbf{v}_x, \mathbf{v}_y, \mathbf{v}_z]^T,
\end{equation}
where $\mathbf{r}$ and $\mathbf{v}$ represent the 3-dimensional Cartesian positions and velocities relative to the origin.
The motion model assumes constant velocity, integrated using a Runge-Kutta 8(7) method~\cite{ref:dormand}.  

The true initial states of the two targets are
\begin{gather}
    \mathbf{x}_{0,1} = [50,\ 50,\ 50,\ 0.5,\ 0.5,\ 2]^T, \\
    \mathbf{x}_{0,2} = [100,\ 100,\ 50,\ -0.5,\ -0.5,\ 2]^T.
\end{gather}
Whereas, the filters are initialized with a single low-weight component to favor births:
\begin{equation}
    w_{0|0}^{(1)} = \num{1d-16},\ \hat{\mathbf{x}}_{0|0}^{(1)} = \mathbf{0}_{6\times 1},\ \hat{\mathbf{P}}_{0|0}^{(1)} = \mathbf{I}_{6\times 6},
\end{equation}
where $\mathbf{0}$ represents the zero vector and $\mathbf{I}$ represents the identity matrix. 

The measurement vector ${\mathbf{z} = [\rho,\ \alpha,\ \varepsilon]^T}$ contains range $\rho$, azimuth $\alpha$, and elevation $\varepsilon$ of the observed target relative to the radar.
Measurements are corrupted by additive zero-mean Gaussian white noise with standard deviations of one\footnote{All units are dimensionless.} for range and $0.5\unit{\degree}$ for both azimuth and elevation. 
Sampling occurs at ${\mathrm{d}t=1}$, starting from 0 and ending at 100.

Let the clutter generated state be $\mathbf{r}_c = [\mathbf{r}_{x,c}, \mathbf{r}_{y,c}, \mathbf{r}_{z,c}]^T$, where it is Poisson distributed in the surveillance region:
\begin{equation}
    \begin{aligned}
        \mathrm{S} = \{ (\mathbf{r}_{x,c}, \mathbf{r}_{y,c}, \mathbf{r}_{z,c} ) \mid 
        &0 \leq \mathbf{r}_{x,c} \leq 200, \; 
        0 \leq \mathbf{r}_{y,c} \leq 200, \; \\
        &0 \leq \mathbf{r}_{z,c} \leq 400 \},
    \end{aligned}
\end{equation}
with an average rate of ${\lambda_{C} = 10 }$ and density ${p_{C} = \num{6.25d-8} }$. 
After being generated, clutter is mapped from the state space into measurements: ${\mathbf{z} = h(\mathbf{r}_c)}$. 

The probability of detection and survival are constant such that ${p_{D} = 0.98}$ and ${p_{S} = 0.99}$, respectively. 
This work does not perform measurement gating, please see~\cite{ref:vo2015}. 

A Gaussian birth model is centered at the crossing point with:
\begin{gather}
    \hat{\mathbf{x}}_{\gamma,k} = [75,\ 75,\ 150,\ 0,\ 0,\ 0]^T,
    \\
    \hat{\mathbf{P}}_{\gamma,k} = \text{diag}([50,\ 50,\ 50,\ 5,\ 5,\ 5]^2), 
\end{gather}
where after each prediction step, ${J_{\gamma,k} = 10}$ components are sampled with weights ${w_{\gamma,k} = 1/100}$.
This model suggests that prior knowledge is given to the system that the two targets will cross.

In this experiment, the EK-PHD filter is used to represent the GM-PHD filter, but will still be referred to as the GM-PHD filter.
Its pruning, merging, and capping parameters are $\num{1d-5}$, $4$, and ${J_{\text{max}} = 250}$, respectively. 
Although the SMC-PHD and EnGM-PHD filters retain a fixed number of components or particles, with ${J_k = 250}$, the GM-PHD filter is allowed to fluctuate its number of components because of its pruning, merging, and capping scheme. 

The SMC-PHD and EnGM-PHD filters both use \textit{k}-means clustering~\cite{ref:vo2005}.

\subsection{Multi-Target Performance Metric}

The Optimal Subpattern Assignment (OSPA)~\cite{ref:schuhmacher2008} metric evaluates tracking accuracy by measuring localization and cardinality errors between estimated states $\hat{\mathbf{x}}$ and ground truth $\mathbf{x}$. 
In this work, the OSPA error metric uses the Euclidean 2-norm, with parameters ${p=2}$ and cutoff ${c=100}$. 
Optimal assignments are performed using the Hungarian, \emph{a.k.a.} Munkres, algorithm~\cite{ref:kuhn1955, ref:munkres1962}.
More details on OSPA can be found in~\cite{ref:vo2006},~\cite{ref:wu2023},~and~\cite{ref:schuhmacher2008}.

\subsection{Results}

Fig.~\ref{fig:trajectories} visualizes the extracted position state estimates, while Fig.~\ref{fig:cardinalities} shows the extracted cardinality estimates. 
Both figures show the results of 250 Monte Carlo simulations as transparent data points with a single run overlaid without transparency for emphasis. 

The GM-PHD filter struggles with the large amount of clutter and overwhelming birth model.  This is likely due to its pruning, merging, and capping scheme, which can lead to the loss of information, impacting the filter's ability to accurately approximate the true intensity function.
This is apparent in Fig.~\ref{fig:trajectories}(a) by the red cloud of extracted state estimates that fill the surveillance region. 
Supporting this reasoning, Fig.~\ref{fig:cardinalities}(a) shows it failing to estimate the correct number of targets.

The SMC-PHD filter also struggles, even with 250 particles.
Figs.~\ref{fig:trajectories}(b)~and~\ref{fig:cardinalities}(b) show that the filter fails to produce position estimates and underestimates the cardinality, indicating that the SMC-PHD filter's Dirac mixture approximation is insufficient. 

The EnGM-PHD filter has the best performance. 
It shows to favorably track the two targets in Fig.~\ref{fig:trajectories}(c) with minimal outliers and consistent cardinality estimates in Fig.~\ref{fig:cardinalities}(c), indicating better approximation of the intensity function.

\begin{figure*}[!ht]
    \begin{subfigure}[t]{0.3333\linewidth}
        \centering
        \includegraphics{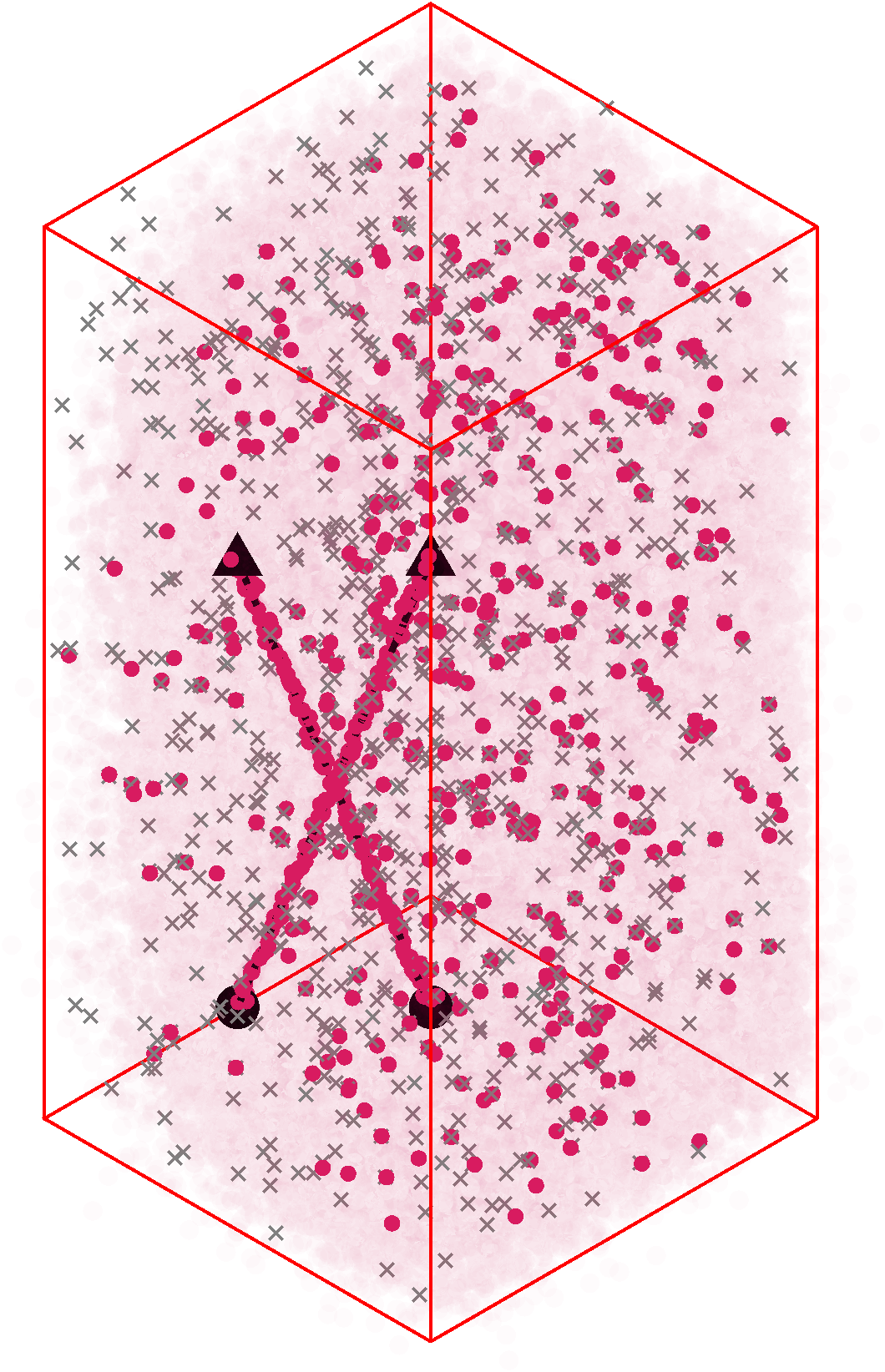}
        \caption{GM-PHD}
    \end{subfigure}%
    ~ 
    \begin{subfigure}[t]{0.3333\linewidth}
        \centering
        \includegraphics{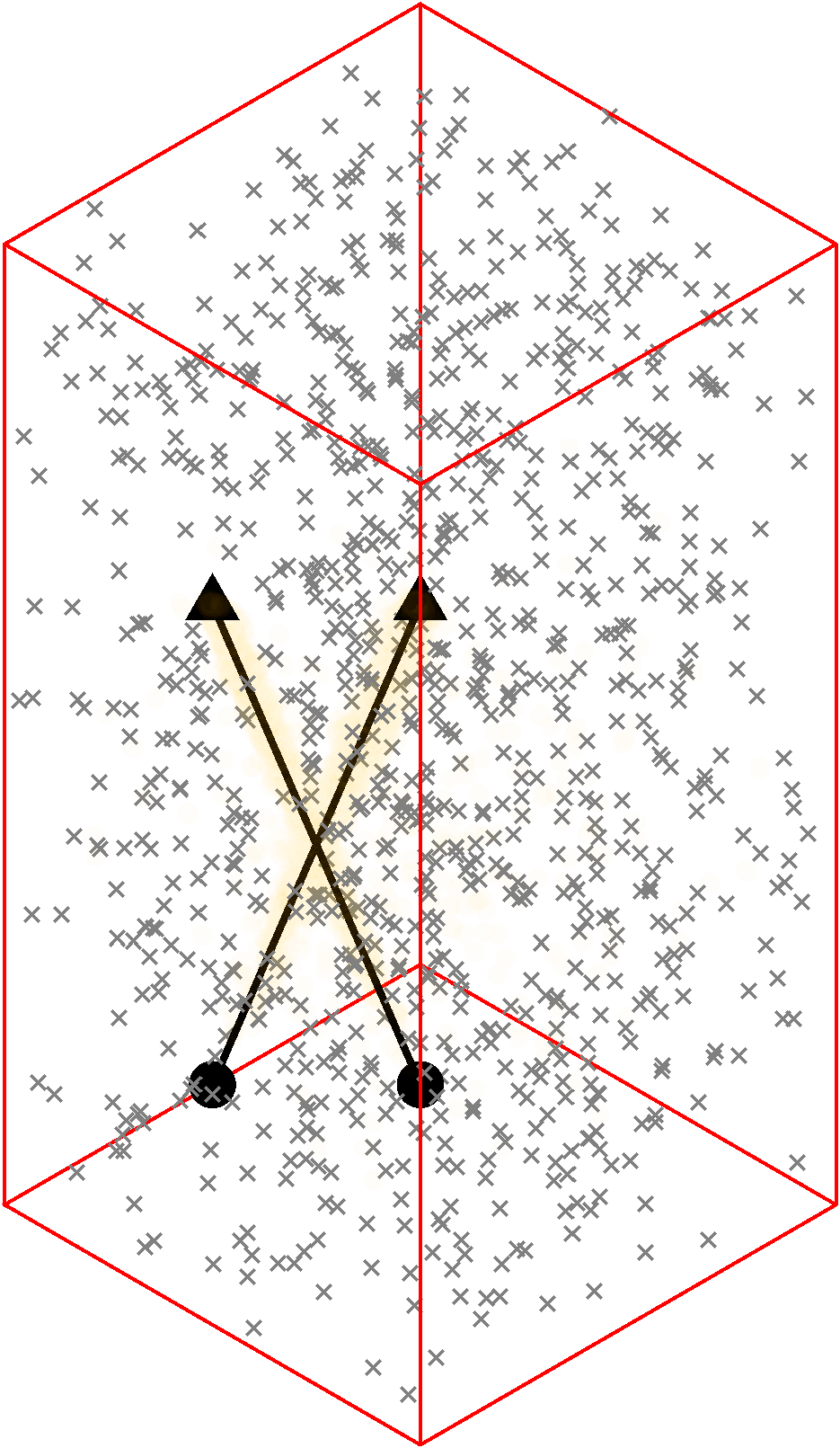}
        \caption{SMC-PHD}
    \end{subfigure}
    ~ 
    \begin{subfigure}[t]{0.3333\linewidth}
        \centering
        \includegraphics{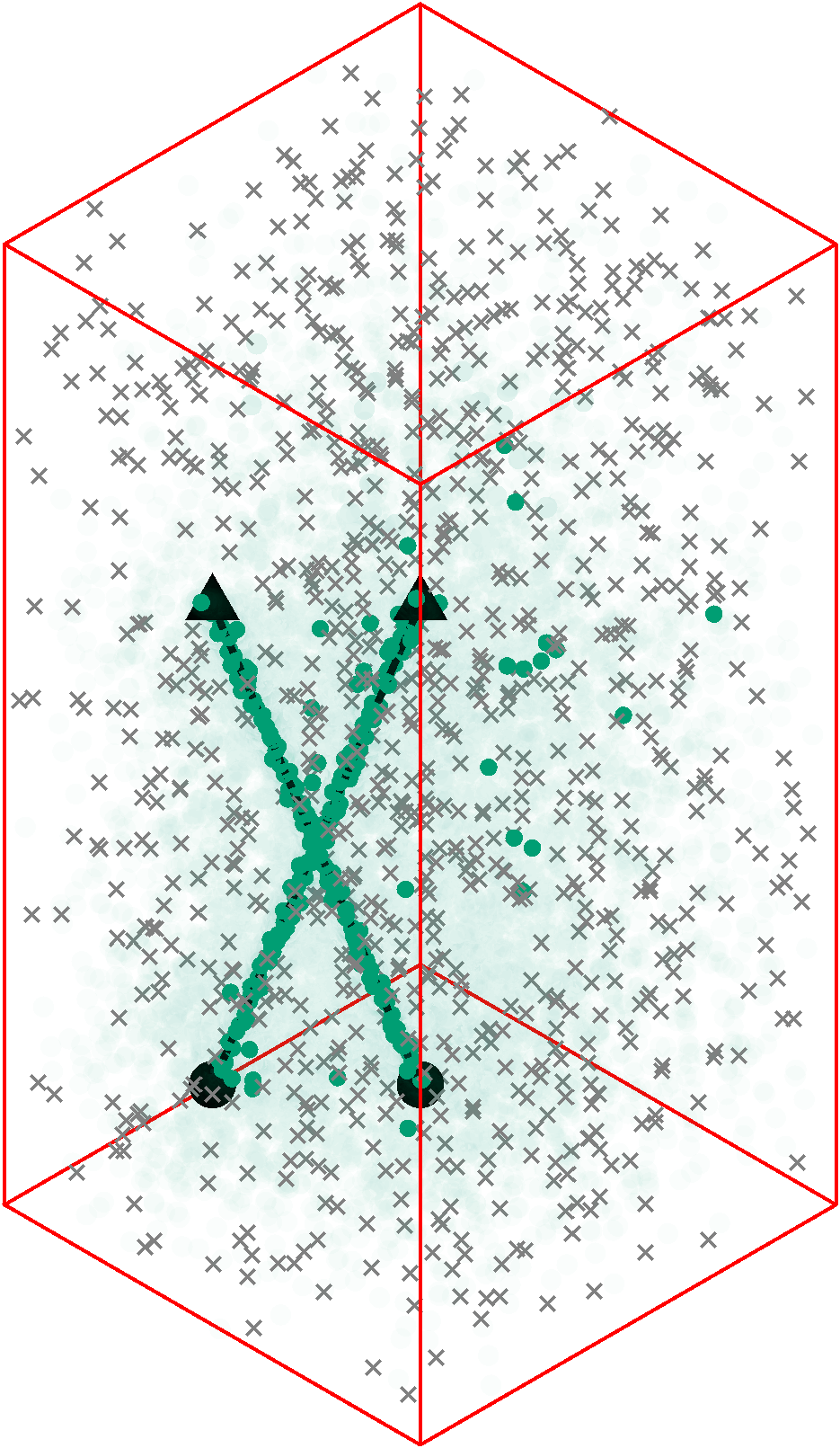}
        \caption{EnGM-PHD}
    \end{subfigure}
    \caption{This figure shows the 3-dimensional true trajectories (black) of the two targets and the extracted position state estimates of the compared filters. 
    The clutter is represented by the gray crosses which are Poisson distributed in the red rectangular surveillance region. These are the results of 250 Monte Carlo simulations, overlaid with a single run emphasized to illustrate. }
    \label{fig:trajectories}
\end{figure*}

\begin{figure*}[!ht]
    \begin{subfigure}[t]{0.3333\linewidth}
        \centering
        \includegraphics{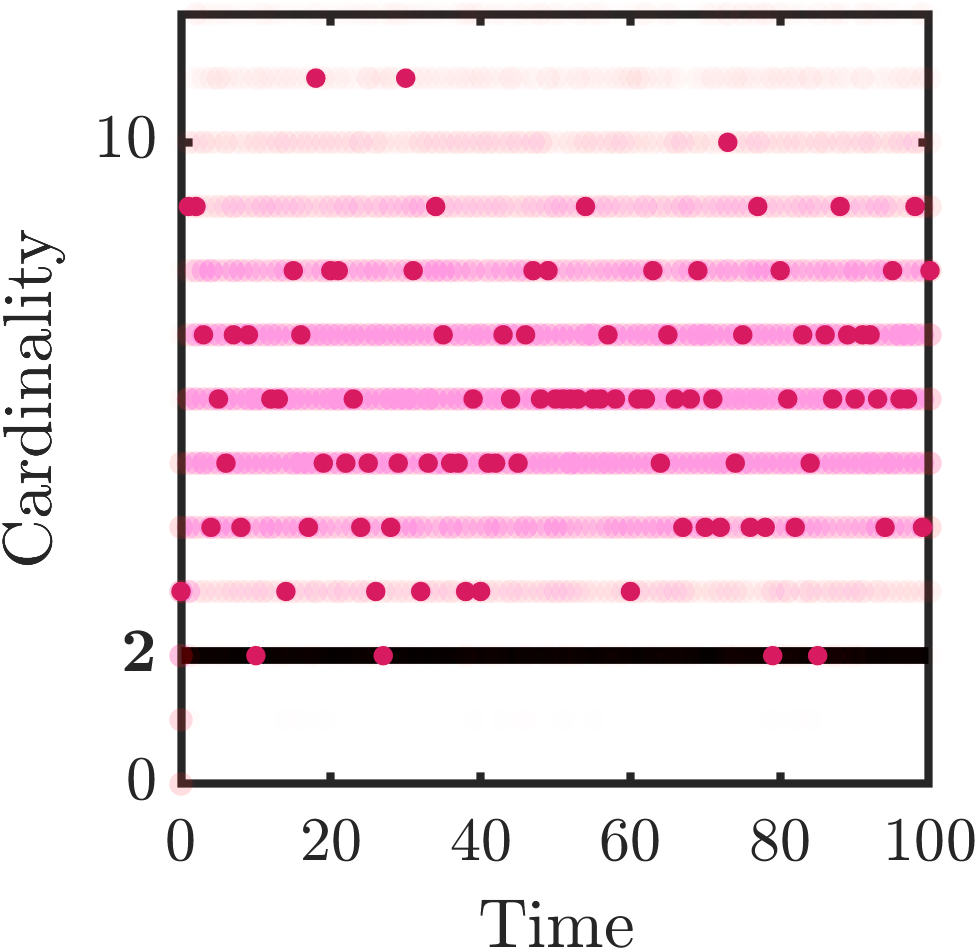}
        \caption{GM-PHD}
    \end{subfigure}%
    ~
    \begin{subfigure}[t]{0.3333\linewidth}
        \centering
        \includegraphics{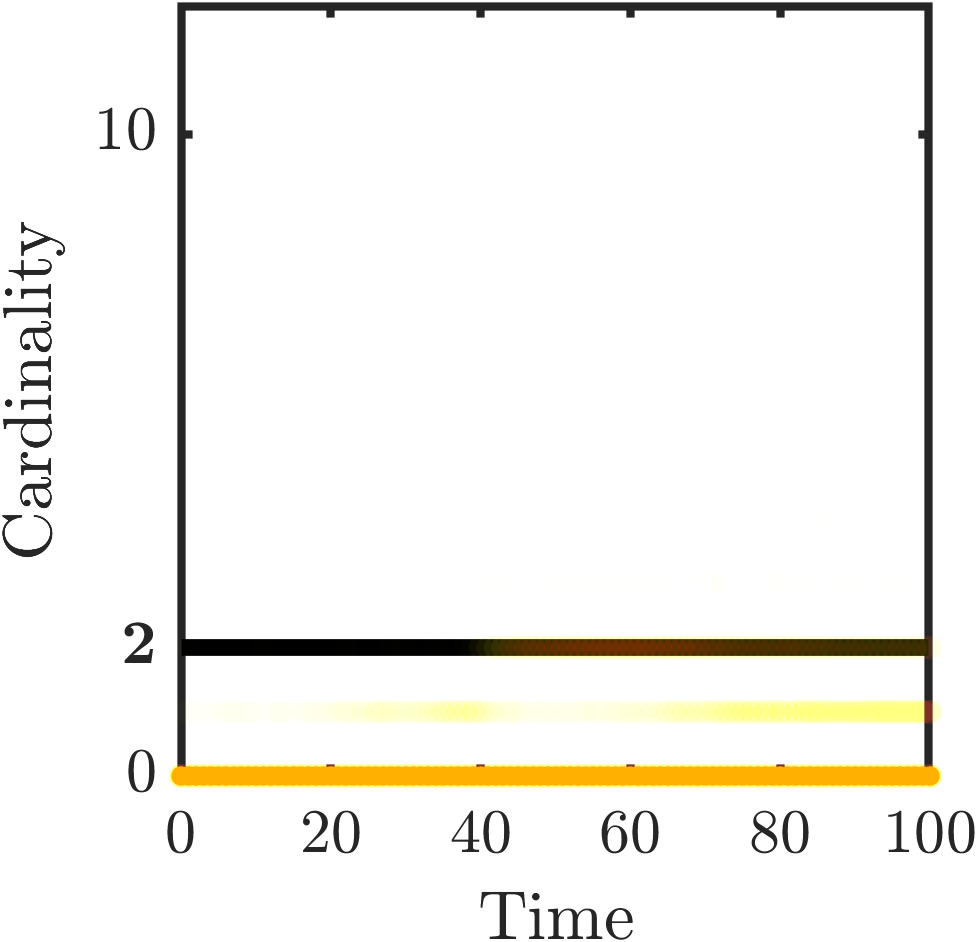}
        \caption{SMC-PHD}
    \end{subfigure}
    ~
    \begin{subfigure}[t]{0.3333\linewidth}
        \centering
        \includegraphics{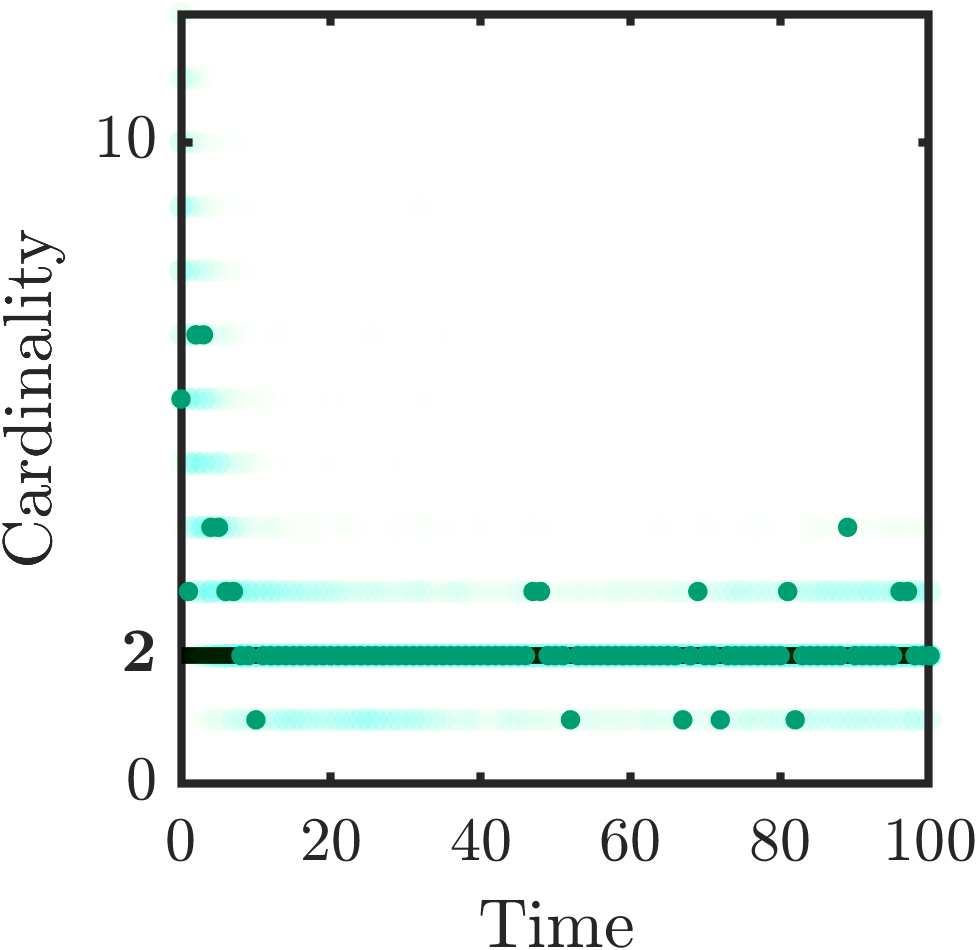}
        \caption{EnGM-PHD}
    \end{subfigure}
    \caption{This figure shows the true cardinality (black) and the extracted cardinality estimates of the compared filters vs. time. 
    These are the results of 250 Monte Carlo simulations, overlaid with a single run emphasized to illustrate. }
    \label{fig:cardinalities}
\end{figure*}

Next, Fig.~\ref{fig:ospa-vs-time} presents the OSPA accuracy of the extracted state estimates in comparison to the truth. 
Corroborating the previous reasoning, the EnGM-PHD filter shows the best overall OSPA, again indicating a better approximation of the intensity function while using the same amount of components or particles. 

\begin{figure}[!ht]
    \centering
    \includegraphics[width=0.5\linewidth]{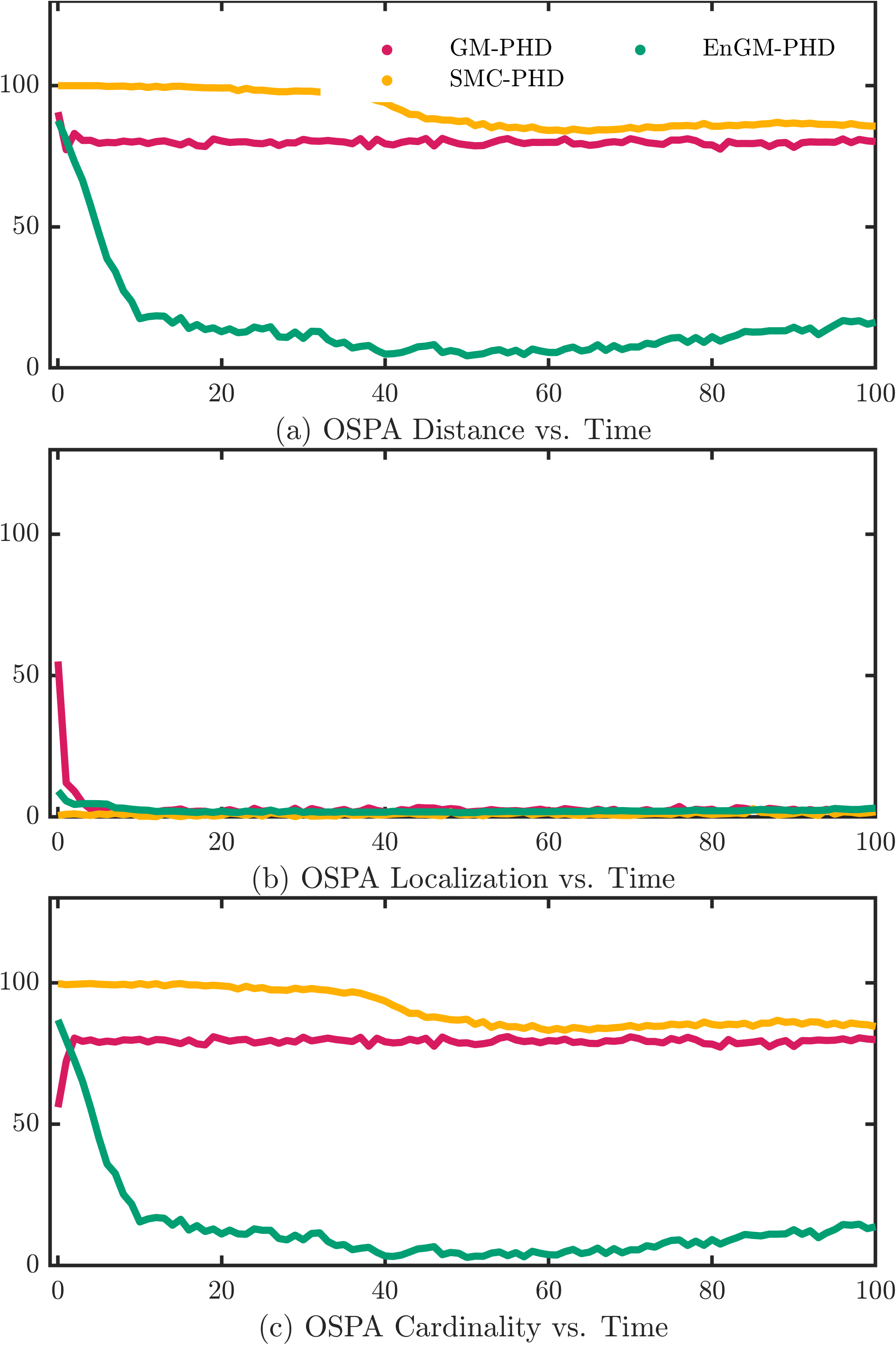}
    \caption{This figure shows the multi-target accuracy (OSPA) of each filter vs. time. 
    Results are averaged over 250 Monte Carlo simulations.}
    \label{fig:ospa-vs-time}
\end{figure}

To evaluate the algorithms relative computational costs, Fig.~\ref{fig:efficiency} compares how each filter's average number of components relate to their simulation wall-clock times.
It should be noted that while the GM-PHD filter is typically more efficient than the SMC-PHD, this is not a guarantee, suggesting that inefficiencies can arise from other processes.
Fig.~\ref{fig:efficiency} indicates this phenomena. 
This is likely due to the GM-PHD filter's merging process, which involves a combinatorial search to identify merging pairs which is computationally expensive and can make the GM-PHD filter slower than the SMC-PHD filter.

Nonetheless, the EnGM-PHD filter runs faster than the GM-PHD filter, even when using \textit{k}-means clustering, demonstrating improved efficiency.
Overall, the proposed EnGM-PHD filter achieves better multi-target filtering performance than both the GM-PHD and SMC-PHD filters while maintaining a similar number of components.

\begin{figure}[!ht]
    \centering
    \includegraphics[width=0.5\linewidth]{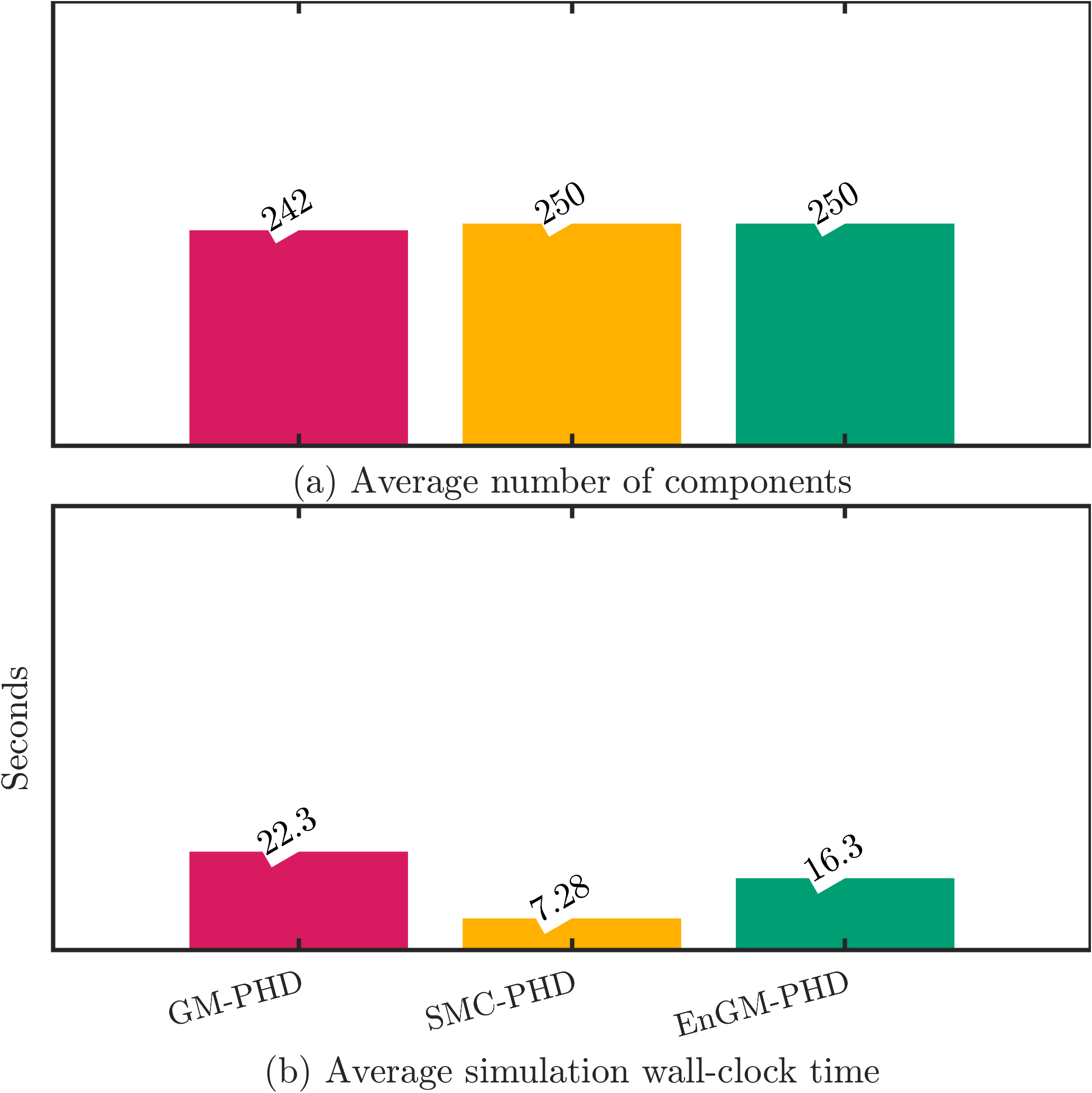}
    \caption{This figure compares the efficiency of the filters and how their number of components relate to their wall-clock times. 
    Results are averaged over 250 Monte Carlo simulations. 
    Using an Intel Core i7 9700K CPU at a base speed of 3.00 GHz and with 16 GB of RAM.}
    \label{fig:efficiency}
\end{figure}

%%%%%%%%%%%%%%%%%%%%%%%%%%%%%%%%%
\section{Conclusions}
\label{sec:conclusion}
%%%%%%%%%%%%%%%%%%%%%%%%%%%%%%%%%

This work proposes the Ensemble Gaussian Mixture Probability Hypothesis Density (EnGM-PHD) filter to produce a smooth and accurate approximation of the intensity function. 
The EnGM-PHD filter combines the Gaussian-mixture-based computational efficiency of the Gaussian Mixture Probability Hypothesis Density (GM-PHD) filter with the particle-based nonlinear capabilities of the Sequential Monte Carlo Probability Hypothesis Density (SMC-PHD) filter. 
The algorithm obtains particles from the posterior intensity function, propagates them through the system dynamics, and then uses Kernel Density Estimation (KDE) techniques to approximate the Gaussian mixture of the prior intensity function.

This approach guarantees convergence to the true intensity function in the limit of the number of components.
Moreover, in the special case of a single target with no births, deaths, clutter, and perfect detection probability, this work proves that the EnGM-PHD filter equations reduce exactly to those of the Ensemble Gaussian Mixture Filter (EnGMF), making the EnGM-PHD filter a straightforward multi-target extension of the single-target EnGMF.

Furthermore, the presented numerical experiment demonstrates better multi-target filtering performance than both the GM-PHD and the SMC-PHD filters, while using the same number of components or particles.  
The EnGM-PHD filter shows promising simulation wall-clock times, but its state extraction method adds a layer of inefficiency and is a limitation of this work.

%%%%%%%%%%%%%%%%%%%%%%%%%%%%%%%%%
%\section*{Appendix}
%%%%%%%%%%%%%%%%%%%%%%%%%%%%%%%%%
%\input{appendix.tex}

%%%%%%%%%%%%%%%%%%%%%%%
%\section*{Appendix}
%%%%%%%%%%%%%%%%%%%%%%%

%%%%%%%%%%%%%%%%%%%%%%%
\section*{Acknowledgments}
%%%%%%%%%%%%%%%%%%%%%%%
This material is based on research sponsored by the Air Force Office of Scientific Research (AFOSR) under agreement number FA9550-23-1-0646, \textit{Create the Future Independent Research Effort (CFIRE)}.

%%%%%%%%%%%%%%%%%%%%%%%
% REFERENCES
%%%%%%%%%%%%%%%%%%%%%%%
\bibliographystyle{unsrtnat}
\bibliography{diss}

\end{document}